\documentclass{article}

\usepackage[preprint,nonatbib]{neurips_2022}
\usepackage{papercommands}
\usepackage{adjustbox}
\usepackage{comment}
\usepackage{multirow}
\usepackage{caption}
\usepackage{float}

\newcommand*\samethanks[1][\value{footnote}]{\footnotemark[#1]}

\usepackage[utf8]{inputenc} 
\usepackage[T1]{fontenc}    
\usepackage{hyperref}       
\usepackage{url}            
\usepackage{booktabs}       
\usepackage{amsfonts}       
\usepackage{nicefrac}       
\usepackage{microtype}      
\usepackage{xcolor}         

\usepackage[linesnumbered, ruled, vlined]{algorithm2e}
\usepackage[inline]{enumitem}

\usepackage{thm-restate}

\title{Pruning Neural Networks via Coresets and Convex Geometry: Towards No Assumptions}

\author{%
  Murad~Tukan\thanks{These authors equally contributed to this paper.}~ \thanks{Department of Computer Science, University of Haifa}\\
  \texttt{muradtuk@gmail.com} \\
  \And Loay~Mualem\samethanks[1] ~\samethanks[2]\\
  \texttt{loaymua@gmail.com} \\
  \And Alaa~Maalouf\samethanks[1] ~\samethanks[2]\\
  \texttt{alaamaalouf12@gmail.com} \\
}

\begin{document}

\maketitle

\begin{abstract}
Pruning is one of the predominant approaches for compressing deep neural networks (DNNs). Lately, coresets (provable data summarizations) were leveraged for pruning DNNs, adding the advantage of theoretical guarantees on the trade-off between the compression rate and the approximation error. However, coresets in this domain were either data-dependent or generated under restrictive assumptions on both the model's weights and inputs. 
In real-world scenarios, such assumptions are rarely satisfied, limiting the applicability of coresets. To this end, we suggest a novel and robust framework for computing such coresets under mild assumptions on the model's weights and without any assumption on the training data. The idea is to compute the importance of each neuron in each layer with respect to the output of the following layer. This is achieved by a combination of L\"{o}wner ellipsoid and Caratheodory theorem.
Our method is simultaneously data-independent, applicable to various networks and datasets (due to the simplified assumptions), and theoretically supported. Experimental results show that our method outperforms existing coreset based neural pruning approaches across a wide range of networks and datasets. For example, our method achieved a $62\%$ compression rate on ResNet50 on ImageNet with $1.09\%$ drop in accuracy.
\end{abstract}

\section{Introduction and Backround}
Deep neural networks (DNNs) achieved state-of-the-art (SOTA) performance on a large variety of tasks, e.g., in computer vision~\cite{he2016deep,krizhevsky2012imagenet}
and natural language processing (NLP;~\cite{radford2019language,devlin-etal-2019-bert}).
However, DNNs usually contain millions or even billions of parameters in order to achieve SOTA performances resulting in large storage requirements and long inference time. This is obstructive when, e.g., dealing with limited hardware or real-time systems such as autonomous cars and text/speech translation.
To this end, a large body of research is dedicated to reducing the size and inference costs of DNNs.


\textbf{Pruning.}
A dominant approach widely used for reducing the size of DNNs is to utilize a pruning algorithm to remove redundant parameters from the original, over-parameterized network. 
In general, pruning can be categorized into two main types: 
(i) Unstructured pruning~\cite{Han15,sipp2019} reduces the number of non-zero parameters by inducing sparsity into weight parameters, which can achieve high compression rates but requires specialized software and/or hardware in order to achieve faster inference times.  (ii) Structured pruning~\cite{he2018soft,liebenwein2020provable,mariet2015diversity} modifies the structure of the underlying weight tensors, by removing filters/neurons from each layer, usually resulting in smaller compression rates while directly achieving faster inference times with no specialized software; see section~\ref{sec:related}

\subsection{Coresets for Pruning}
Notably, many recent papers focused on various types of filter pruning~\cite{renda2020comparing, molchanov2019importance} potentially due to the empirical observation that existing filter pruning approaches consistently yield impressive results. However, most pruning methods are based on heuristics, lacking theoretical guarantees on the trade-off between the compression rate and the approximation error. 
This was the motive for introducing coresets~\cite{mussay2021data,liebenwein2020provable} to the world of pruning. 

\textbf{Coresets.} In machine learning, we are (usually) given an input set $P\subseteq\REAL^d$ of $n$ points, its corresponding weights function $w:P \to \REAL$, a feasible set of queries $X$, and a loss function $\phi:P \times X \to [0,\infty)$. The tuple $(P,w,X,\phi)$ is called \emph{query space}, and it defines the optimization problem at hand. 
For a given problem that is defined by its query space $(P,w,X,\phi)$, and an error parameter $\eps\in(0,1)$, an $\eps$-coreset is is a small weighted subset of the input points that approximates the loss of the input set $P$ for every feasible query $x$, up to a provable bound of $1+\eps$.

Since coresets approximate the cost of every query, traditional (possibly inefficient) algorithms/solvers can be applied on coresets to obtain an approximation of the optimal solution on the full data, using less time and memory; see Section~\ref{weak-coresets} in the appendix for more details.

\begin{figure}[htb!]
    \centering
    \includegraphics[width=\linewidth]{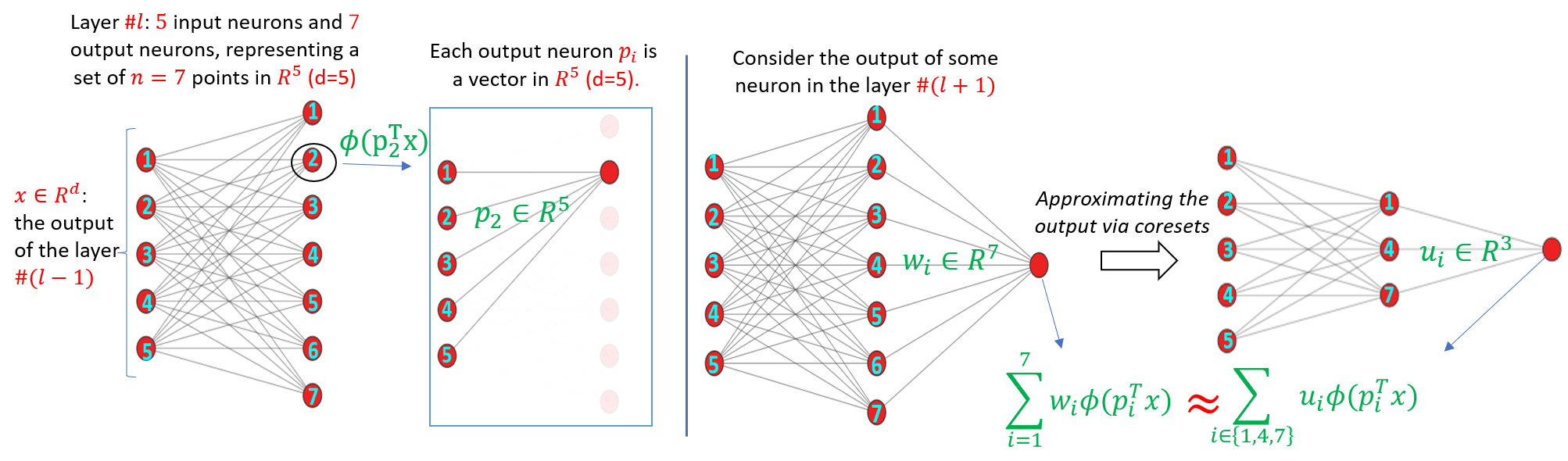}
    \caption{ Illustration of our neuron coreset construction on a toy example.}
    \label{fig:diagram}
\end{figure}

\textbf{Pruning via coresets.} Recently, some inspiring innovative frameworks~\cite{mussay2021data,liebenwein2020provable,baykal2018datadependent,bhatt2021} leveraged the idea of coresets for pruning DNNs. Any layer $\ell$ can be represented as a set $P=\br{p_1,\cdots,p_n}$ of $n>1$ points in $\REAL^d$, where $d$ is the number of input neurons, and $n$ in the number of output neurons, i.e., each point $p\in P$, represents a specific neuron using its $d$ weights (parameters). When the layer receives an input vector $x\in \REAL^d$, it outputs the vector $(\phi(p_1^Tx),\cdots,\phi(p_n^Tx))$, where $\phi:\REAL\to \REAL$ is an activation function which defines a non-linear mapping. 
Focusing on a single neuron $\eta$ in the layer that follows $\ell$, defined by its corresponding vector of weights $w=(w_1,\cdots,w_n)$, we set in the context of coresets, $w(p_i):=w_i$ for every $i\in \br{1,\cdots,n}$ - this is just a mapping from $p_i$ to $w_i$ to simplify the writing and reading. Note that the output of this neuron is $\sum_{p\in P} w(p)\phi\term{p^Tx}$. Assuming that we are given an $\eps$-coreset $(C,u)$ for the query space $\term{P,w,\REAL^d,\phi}$ where $C\subseteq P$, and $u:C\to \REAL$, we have that $(C,u)$ approximates the output of this specific neuron $\eta$ for every query using less parameters; see Figure~\ref{fig:diagram}. To formalize the stated above, we now define coresets in the context of activation functions.

\begin{definition}[Coreset for activation functions]
\label{def:coreset_activation}
Let $\eps\in(0,1)$, and let $\term{P,w,\REAL^d, \phi}$ be a query space. 
Then the pair $\term{C,u}$, is an $\eps$-coreset for $\term{P,w,\REAL^d, \phi}$ if (i) $C \subseteq P$, (ii) $u : C \to [0,\infty)$, and (iii) for every $x \in X$, $\abs{1 - \frac{\sum_{q \in C} u(q)\phi\term{q^Tx}}{\sum_{p \in P} w(p)\phi\term{p^Tx}}} \leq \eps.$
\end{definition}

Since $C$ is a subset of $P$, we can remove (assign zero to) all weights from $w$ that corresponds to points not chosen to be in $C$ from $P$, and replace the weights of the chosen points in $C$ with the new weights vector $u$; see Figure 1 in~\cite{mussay2021data} for a visual illustration. \textbf{To prune neurons,} we refer the reader to Section~\ref{sec:ext} as it is a simple extension. Prior work showed that such approaches successfully result in high compression rates across a wide range of networks and datasets, and even achieves SOTA performance on a verity of them.

The main (strong) advantage of the coreset approach over others was the provided provable theoretical guarantees on the tradeoff between the compression rate and the approximation error, which supports worse case scenarios. In addition, coresets play an important role in improving the generalization properties of the trained networks~\cite{baykal2018datadependent,mirzasoleiman2020coresets}. 

\textbf{Sensitivity sampling for constructing pruning coresets.} To compute such coresets, both~\cite{liebenwein2020provable,mussay2021data} utilised the known sensitivity sampling framework~\cite{braverman2016new,langberg2010universal}. In short the sensitivity of a point $p\in P$ in some query space $\term{P,w,X,\phi}$ corresponds to the importance of this point with respect to the query space at hand, and it is defined as $s(p) = \sup_{x\in X}\frac{w(p)\phi(p,x)}{\sum_{q\in P}w(q)\phi(q,x)}$ - where the denominator is not equal to zero. Once we bound these sensitivities, we can sample points (neurons) from $P$ according to sensitivity bounds, and re-weight the sampled points to obtain a coreset. The size of the sample is proportional to the sum of these bounds.
 See Section~\ref{detailesonsensitity} and Theorem~\ref{thm:braverman_coreset} for more details in the Appendix.




\subsection{Our contribution}
Prior coreset methods for pruning DNNs either (i) imposed restrictive assumptions both on the model's weights and inputs~\cite{mussay2021data}, i.e., the input set $P$ representing the neurons, and the query set $X$ which represents the inputs of the layer, are enclosed in a ball in $\REAL^d$ of radius $r_1$ and $r_2$, respectively, or (ii) the methods are data-dependent, i.e., use a mini-batch of the input set to measure the influence of each parameter on the loss  function~\cite{baykal2018datadependent,liebenwein2020provable}. 

To this end, in this work, we take coresets a step further into the realm of pruning by introducing a unified framework with provable guarantees for pruning DNNs (weights and neurons/filters) while minimally affecting the generalization error. Our main improvement is that our framework is simultaneously (i) \textbf{data-independent}, (ii) \textbf{requires a single assumption} on the model's weights, and (iii) \textbf{provably guarantees a multiplicative factor approximation}, which is favourable upon additive approximations; see Theorem~\ref{thm:relu_core}. 
The approach is based on the widely used theory of coresets allowing us to suggest a provable guarantee on the tradeoff between the approximation error and compression rate for each layer. 

We conducted experimental results which established new SOTA benchmarks for structured pruning via coresets across a wide range of networks and datasets. We share all of our resulted models~\cite{opencode}.


\section{Method}
\label{sec:method}
In general, the coreset (for pruning) technique hinges upon the insight that any linear layer such as convolutions, can be casted as a matrix multiplication~\cite{mussay2021data}. Hence, we focus in what follows on fully connected (FC) layers, while the details holds for any linear layer.
Furthermore, for simplicity, we assume in what follows that the weights of $P$ are all equal to $1$ and thus our query space is denote by $\term{P,\REAL^d, \phi}$, and the sensitivity of a point $p\in P$ is simply $s(p) = \sup_{x\in X}\frac{\phi(p,x)}{\sum_{q\in P}\phi(q,x)}$. Note that our proofs are easily extended to the general case where we are given a weight function $w:P\to \REAL$ as discussed in Section~\ref{sec:ext}.

\subsection{Preliminaries}
\label{sec:preliminaries}
\textbf{Notations.} For a positive integer $n$, we use $[n]$ to denote the set $\{1,\ldots,n\}$. For $c \in \REAL^d$ and a symmetric positive definite matrix $G \in \REAL^{d \times d}$, we define $E\term{G,c}:=\br{x \in \REAL^d \middle| \term{x-c}^T G \term{x-c} \leq 1}$ to be the ellipsoid defined by $c$ and $G$. For an ellipsoid $E\term{G,c}$, each endpoint of a semi principal axis is called a vertex of $E\term{G,c}$. We define $\mathrm{rank}\term{P}$ for any set $P \subseteq \REAL^d$ to be the dimension of the affine subspace that $P$ lies on. For a set $P\subset \REAL^d$ the convex hull of $P$ is denoted by $\conv{(P)}$. Finally, vectors are treated as column vectors.

\subsection{Novelty - L\"{o}wner ellipsoid meets Carath\'{e}odory}
\begin{figure}[t!]
    \centering
    \includegraphics[width=\linewidth]{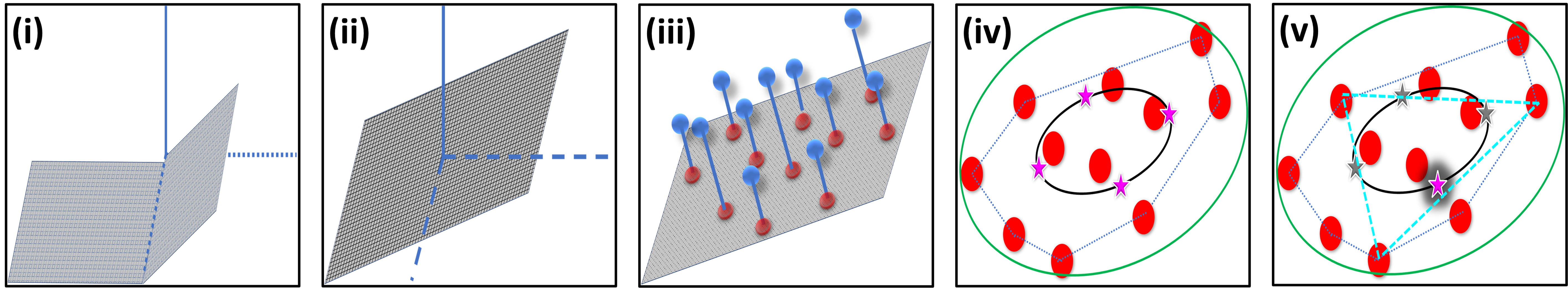}
    \caption{Our novelty in a nut-shell: The very first steps of our technique rely on bounding the ReLu activation function by the $\ell_1$-regression loss function, e.g., for $\relu{p^Tx}$, where $p=(1,0)$ in this example (shown in (i)), we first bound it by the $\ell_1$-regression loss function (shown in (ii)) using Definition~\ref{def:complexity_mes}. Following this step, a set of points $P$ can be projected into a low dimensional subspace of dimension $\mathrm{rank}\term{P}$ using any dimensionality reduction algorithm as presented in (iii), resulting in the set $P^\prime$. (iv) The convex hull (blue dashed lines) $P^\prime$ (red points) is enclosed by its L\"{o}wner ellipsoid (depicted in green). (v) Finally, for each vertex (magenta star) of the enclosed ellipsoid (black ellipsoid), its Carath\'{e}dory set is found (red points connected by cyan dashed lines).}
    \label{fig:my_label}
\end{figure}
Our method hinges upon a combination of two known tools from convex geometry. 
The novelty of our approach exploits the following observation. Most activation functions $\phi$ are continuous non-decreasing functions, which indicate that for every query $x$ and a set of points $P$, the maximal contribution to $\sum_{p\in P} \phi(p^Tx)$ with respect to such activation function is associated to a point on the convex hull of $P$. 
By finding a geometrical body $B$ of bounded number of vertices, that is (i) enclosed in $\conv\term{P}$ and (ii) with some dilation factor (expanding) $B$ enclose $\conv\term{P}$, we will be able to represent each point $p$ on the boundary of the convex hull of $P$ as a convex combination of two points $p_1,p_2$, one of each ($p_1$) on $B$ and the second ($p_2$) on its dilated form, which is formalized as the set $\left\lbrace \alpha (x - c) + c \mid x \in B, \forall \alpha \in [0, 1] \right\rbrace$, where $c$ here denotes the center of $B$. For such task, L\"{o}wner ellipsoid is leveraged. 

\begin{theorem}[John-L\"{o}wner ellipsoid\cite{John14}]
\label{thm:loewner:ellipsoid}
Let $L\subseteq\REAL^d$ be a set of points such that the convex hull of $L$ has a nonempty interior. Then, there exists an ellipsoid $E\term{G,c}$ (also known as the \emph{MVEE}), where $G\in\REAL^{d\times d}$ is a positive definite matrix and $c \in \REAL^d$, of minimal volume such that $\frac{1}{d} \term{E\term{G,c}-c}+c \subseteq \conv\term{L} \subseteq E\term{G,c}.$ 
If $L$ is symmetric around the center $c$, then the dilation factor can be reduced to $\frac{1}{\sqrt{d}}$.
\end{theorem}

\begin{minipage}[t]{0.46\textwidth}
  \vspace{0pt}  
\begin{algorithm}[H]
\caption{$\infcoreset\term{P,m}$}
\label{alg:mainINF}
 \SetKwInOut{Input}{Input}
\SetKwInOut{Output}{Output}
\SetAlgoLined
\DontPrintSemicolon
\Input{$P\subseteq\mathbb{R}^d$ of $n$ points with rank $r$}
\Output{A subset $S \subseteq P$ that satisfies Lemma~\ref{lem:l1_reg}}
$(Y,z) := $ affine subspace that $P$ lies on, i.e. $P \subseteq \br{xYY^T + z \middle| x \in \REAL^d}$ \label{alg_inf:affine}\;
$P^\prime := \br{p^\prime := pY^T}$ \label{alg_inf:Pprime}\;
Let $\mathrm{map} : P^\prime \to P$ a function that maps every $p^\prime\in P^\prime$ to its corresponding point $p\in P$ \label{alg_inf:map}\;
$S := \emptyset; K := \emptyset$\;
$E\term{G,c} := \mvee\term{\conv\term{P^\prime}}$ \label{alg_inf:Ellipsoid}\;
$V := $ the vertices of ellipsoid $\frac{1}{r} \term{E\term{G,c} - c} + c$ \label{alg_inf:V}\;
\For{every $v \in V$\label{alg_inf:CaraStart}}{
$K := K \cup \cara\term{v, P^\prime}$ \label{alg_inf:Cara} \tcp{See Algorithm~\ref{alg:cara} in the supplementary material.}
}\label{alg_inf:CaraEnd}
$S := \br{\mathrm{map}\term{q} \middle| q \in K}$ \label{alg_inf:S}\;
\Return $S$
\end{algorithm}
\end{minipage}
\hfill\vline\hfill
\begin{minipage}[t]{0.46\textwidth}
\vspace{0pt}
\begin{algorithm}[H]
\caption{$\coreset\term{P,m}$}
\label{alg:main}
 \SetKwInOut{Input}{input}
\SetKwInOut{Output}{output}
\SetAlgoLined
\DontPrintSemicolon
\Input{A set $P\subseteq\mathbb{R}^d$ of $n$ points, and a sample size $m$}
\Output{ A weighted set $(C,u)$}
$Q := P$, $i := 1$, $C := \emptyset$\label{alg1:line_i}\;
\While{\label{line:main_while_loop}$\abs{Q} \ge 2\mathrm{rank}\term{Q}^2$}{
$S_i := \infcoreset\term{Q}$\label{alg1:Si}\;
\For{every $p \in S_i$ \label{alg1:sens_start}}{
$s(p) := \frac{2\mathrm{rank}\term{Q}^{1.5}}{i}$ \label{alg1:used_i}\;
} \label{alg1:sens_end}
$Q :=  Q\setminus S_i$, $i := i+1$ \label{alg1:line_advance_i}\;
}
\For{every $p \in Q$ }{
$s(p) := \frac{2\mathrm{rank}\term{Q}^{1.5}}{i}$\;
}
$t := \sum_{p\in P}s(p)$\;
$C:=$ an i.i.d sample of $m$ points from $P$, where each  $p\in P$ is sampled with probability $\frac{s(p)}{t}$. \label{alg1:start_coreset}\;
$u(p):= \frac{t}{m\cdot s(p)}$ for every $p\in C$ \label{alg1:end_coreset}\;
\Return $(C,u)$
\end{algorithm}
\end{minipage}

Afterwards, $p_1$ and $p_2$ should be represented by points from $P$. Each point on $B$ (specifically, $p_1$) can be represented via a convex combination of $d+1$ points from $P$. The same holds for points on the dilated form of $B$ (e.g., $p_2$) but via a conical combination (linear combination where the weights are non-negative and the sum of weights is not necessarily $1$). This problem is solved by invoking Carath\'{e}odory theorem.


\begin{theorem}[\cite{Caratheodory07,Steinitz13}]
\label{thm:cara}
For any $A\subset\mathbb{R}^d$ and $p\in\conv(A)$, there exists $m \le d+1$ points $p_1,\ldots,p_m\in A$ (denoted by a Carath\'{e}odory set of $p$) such that $p\in\conv\term{\br{p_1,\ldots, p_m}}$.
\end{theorem}

Finally, it is known that some functions, including the ReLU function, do not admit an $\eps$-coreset of size $o(n)$~\cite{mussay2021data,munteanu2018coresets}. Thus, we use a generalized form of what is known as the complexity measure of a set of points, which was first introduced in~\cite{munteanu2018coresets} and later leveraged in~\cite{mai2021coresets}. This measure is used to determine the complexity of a given set $P$ with respect to ReLU, and the coreset size theoretically.

\begin{definition}[Regression Complexity Measure]
\label{def:complexity_mes}
Let $P \subseteq \REAL^d \times \br{1}$, the regression complexity measure of $P$ is defined as $\mu\term{P} = \sup_{x \in \REAL^{d+1}} \frac{\sum_{q \in \br{p \in P \middle| p^Tx\leq 0}} \abs{q^Tx}}{\sum_{q \in \br{p \in P \middle| p^Tx > 0}} \abs{q^Tx}}$, where the denominator is $\geq 0$, and the last entry of every $p \in P$ is $1$, reserved for the bias/intercept term.
\end{definition}


\subsection{Our Pruning Scheme}
In what follows, we present our data summarization technique for ReLU on the dot product function.  Then in Section~\ref{sec:ext}, we discuss that our results can be easily extended to a wide family of activation functions including the Sigmoid function, as recently shown in~\cite{mai2021coresets}. 
First we present Algorithm~\ref{alg:mainINF}, which serves as a stepping stone towards bounding the sensitivities. 


\textbf{Overview of Algorithm~\ref{alg:mainINF}.} The algorithm receives as input a set $P \subset \REAL^d$ whose rank is $r \in [d]$ and deterministically finds a  subset $S \subseteq P$ which satisfies that for every $j \in [d]$, $X \in \REAL^{d \times j}$ and $v \in \REAL^d$, $\frac{\max_{p \in P} \norm{\term{p-v}X}_1}{\max_{p \in S} \norm{\term{p-v}X}_1} \leq r^{1.5}$. To do so, first, we find the affine hyperplane that $P$ lies on, followed by computing the low dimensional representation of $P$, denoted by $P^\prime$; see Lines~\ref{alg_inf:affine}--\ref{alg_inf:Pprime}. Note that if $\mathrm{rank}(P)=d$, then we can either keep $P$ as it is (i.e., $P':=P$), or use dimensionality reduction tricks as detailed in Section~\ref{sec:ext}. To compute the output $S$, we first bound the convex hull of $P^\prime$ by its L\"{o}wner ellipsoid $E(G,c)$ in Line~\ref{alg_inf:Ellipsoid}, followed by computing the dilated ellipsoid of $E(G,c)$, namely, $\frac{1}{r}\term{E - c} + c$. Let $V$ be the set of vertices of such ellipsoid; see Line~\ref{alg_inf:V}. Now, for each point $v \in V$, we represent it as a convex combination of $r+1$ points from $P^\prime$ via Theorem~\ref{thm:cara}, and store the union of such sets (each of size at most $r+1$) of points into $K$ as done in Lines~\ref{alg_inf:CaraStart}--\ref{alg_inf:CaraEnd}.
For each point in $K$, we map it back to $\REAL^d$ to satisfy Lemma~\ref{lem:l1_reg}. 
To sum up Algorithm~\ref{alg:mainINF}, we observe that the vertices can be used via canonical combinations with their dilated form to describe every point on the convex hull of the input data (in our end, it would the network's weights). Hence the Carath\'{e}odory set of these vertices from the input points lying on the convex hull can be further used to also represent points lying on the convex hull. This is the core idea which enable us in forming our $\ell_\infty$ coreset for any $\ell_\rho$ regression problem where $\rho \in (0, \infty)$.

We now discuss Algorithm~\ref{alg:main} which is responsible for constructing an  $\eps$-coreset with respect to activation functions. Its input is a set $P \subset \REAL^d$ and a sample size $m \geq 1$. 

\textbf{Overview of Algorithm~\ref{alg:main}.} First set $Q := P$. At each iteration $i$, the algorithm obtains a subset $S_i \subseteq Q$ as an output to a call to $\infcoreset(Q)$ as stated in Line~\ref{alg1:Si} of Algorithm~\ref{alg:mainINF} such that for every $x \in \REAL^d$, it holds that 
$\frac{\max_{p \in P} \abs{p^Tx}}{\max_{p \in S_i} \abs{p^Tx}} \leq r^{1.5}$ with $r$ being the rank of $Q$. The sensitivity of each point in $S_i$ is bounded from above by $\frac{2d^{1
.5}}{i}$ as stated in Lines~\ref{alg1:sens_start}--\ref{alg1:sens_end}.
The idea behind these bounds lies in our proof of Theorem~\ref{thm:relu_core}. The set $S_i$ is removed from $Q$, and this procedure is repeated with respect to $Q$ until the size of $Q$ is small enough. The obtained sensitivities are the ones needed for computing the pruning coresets. Finally, we utilize the sensitivity sampling framework of~\cite{braverman2016new} to obtain the desired coreset; see Lines~\ref{alg1:start_coreset}--\ref{alg1:end_coreset}.

\subsection{Analysis}\label{sec:analysis}
In this section, we prove the correctness of our algorithms. 
The following lemma shows that for each point $p$ that is inside some convex hull $S$, its $\ell_1$ distance to any affine subspace is always bounded from above by $\ell_1$ distance from the same affine subspace, of some other point $q \in S$.
\begin{restatable}{lemma}{lemconvineq}
\label{lem:conv:ineq}
Let $d,\ell,m \geq 1$ be integers. Let $p \in \REAL^d$ and $A \subseteq\REAL^d$ be a set of $m$ points with $p \in \conv(A)$ so that there exists $\alpha:A \to [0,1]$ such that $\sum_{q \in A} \alpha(q)=1$ and $\sum_{q\in A}\alpha(q) \cdot q=p$.
Then for every $ Y \in \REAL^{d \times \ell}$ and $v \in \REAL^{\ell}$, $\norm{(p - v)Y}_1 \leq \max_{q \in A}\norm{\term{q - v}Y}_1$.
\end{restatable}

The following states the provable guarantees of Algorithm~\ref{alg:mainINF}.
\begin{restatable}[$\ell_\infty$-coreset for $\ell_1$-regression]{lemma}{linflonereg}
\label{lem:l1_reg}
Let $P \subseteq \REAL^d$ be a set of points, and $r$ be the rank of $P$. Let $j \in [d-1]$ and let $S$ be the output of a call to $\infcoreset(P)$. Then
\begin{enumerate*}[label=(\roman*)]
    \item $\abs{S} \in O\term{r^2}$, and \label{clm:l_inf_1} 
    \item for every $X \in \REAL^{d \times j}$ and $v \in \REAL^d$, $\frac{\max_{q \in P} \norm{\term{q - v}X}_1}{\max_{q \in S} \norm{\term{q - v}X}_1} \in \left[1, 2r^{1.5}\right]$. \label{clm:l_inf_2}
\end{enumerate*}
\end{restatable}

The following theorem states our main result.

\begin{theorem}[ReLU $\eps$-coreset]
\label{thm:relu_core}
Let $P \subseteq \REAL^d$, $\eps,\delta \in (0,1)$, $r =\mathrm{rank}\term{P}$, $\mu\term{P}$ be as in Definition~\ref{def:complexity_mes}, and let $m \in O\term{\frac{\mu\term{P} r^{3.5}\log{n}}{\eps^2} \term{d\term{\log\term{\mu\term{P} r\log{n}}} + \log\term{\frac{1}{\delta}}}}$. Let $(C,w)$ be the output of a call to $\coreset(P,m)$; see Algorithm~\ref{alg:main}. Then, with probability at least $1 - \delta$, $\term{C,w}$ is an $\eps$-coreset for $\term{P, \REAL^d, \relu{\cdot}}$; see Definition~\ref{def:coreset_activation}. 
\end{theorem}



\textbf{Time analysis.} Letting $r$ be the rank of $P$, the time complexity of Algorithm~{1} can be dissected to two main parts: (i) Computing the L\"{o}wner ellipsoid in $O(nr^2\log{n})$ time using the method proposed in~\cite{todd2007khachiyan} and (ii) computing the Carath\'{e}dory set in $O(nr + r^4\log{n})$ time via~\cite{maalouf2019fast}. Since $V$ can contain up to $O(r^2)$ points, the overall time for Algorithm~\ref{alg:mainINF} is $O(nr^2\log{n} + nr^3 + r^6\log{n})$. As for Algorithm~\ref{alg:main}, it takes $O(n^2r^2\log{n} + nr^4\log{n})$. Indeed, as explained in Section~\ref{sec:ext}, a dimensionality reduction algorithm may be applied to improve the run time (reducing the $r^6$ factor). Furthermore, the run time of our algorithm can be improved, using the merge-and-reduce tree from the literature of coresets to reduce the $n^2$ terms to $n\log{n}$, i.e., the running time can be reduced to $O(n\log^2(n)r^2 + nr^4)$. For a data-independent provable method, this running time is reasonable. 

\textbf{Our advantages over previous results.} Our coreset supports different activation functions without the need to change the sensitivity that much. Specifically, it will only be multiplied by some scalar, unlike previous coresets where different losses impose drastically different sensitivities/leverage scores and algorithms. This is since our coreset unlike other coresets is in its essence a framework of coresets for different $\ell_\rho$ losses, as it can be used as is for different $\ell_\rho$ losses and yet still attain $\epsilon$-approximation. In addition, when the rank of the input points is small, then our method outperforms previous methods.If the input data is of full rank, previous methods obtain faster coresets construction.

\textbf{On the boundness of the regression complexity measure.} First of all, there exists an example where the complexity measure is unbounded, e.g., consider a set of points distributed evenly on a unit ball. In this case, you can always find a point where a hyperplane separating it from the rest of the points can be found such that the one half-space of this hyperplane contains only this point while the other half-space contains the rest of the points. This leads to an infinite complexity measure. Such an example is also mentioned in~\cite{mussay2021data}, when assessing the hardness of generating multiplicative-approximation coresets for ReLU functions.

Theoretically, the complexity measure is influenced by how free can the bias term be (the last entry of x); see Definition~\ref{def:complexity_mes}. This term is the only thing that can ensure that one point can be separated from the rest in the sense of finding a separating hyperplane, leading to an infinite complexity measure. Bounding on this term, leads to bounded complexity measure from a theoretical point of view.

In the context of model pruning, from the perspective of the complexity measure, the model's weights are the input denoted by a matrix $P \in \REAL^{n \times (d+1)}$, while the query is now $\REAL^d \times \br{1}$. Thus the complexity measure is now $\mu\term{P} := \sup_{x \in \REAL^d \times \br{1}} \frac{\sum_{q \in \br{p \in P \middle| p^Tx \leq 0}} \abs{q^Tx}}{\sum_{q \in \br{p \in P \middle| p^Tx > 0}} \abs{q^Tx}}$. With this in mind, we observe that the complexity measure is now an instance of the complexity measure used in~\cite{mai2021coresets}. The complexity measure now relies entirely on the structure of the model's weights, where the goal is to find the largest ratio between the sum of the absolute of the values inside the rectified neurons prior to applying the rectification, and the sum values of non-rectified neurons. To bound this measure, we can use a variant of the algorithm described in the proof of Theorem 3 in~\cite{munteanu2018coresets}.

\subsection{Extensions}
\label{sec:ext}
Our suggested scheme can be extended to support many other variants of the pruning problem.

\textbf{Various activation functions.} 
Our result can be extended to a family of activation functions called \say{Nice hinge functions}; see Definition~\ref{def:nice_hinge}. Let $(P,X,\phi)$ be a query space,  where $\phi$ is a \say{Nice hinge functions}. 
To bound the sensitivity of a point $p$, we first bound the nominator of $s(p)$ by proving that $\forall x\in X:\phi(p^Tx)<\abs{p^Tx}$. For bounding the denominator from below, recently~\cite{mai2021coresets} proved that $\forall x\in X: \sum_{p\in P}\relu{p^Tx} \lessapprox \sum_{p\in P}\phi{(p^Tx)}$; see full detail in Section~\ref{sec:ext_other_activations}.

\textbf{Weighted Input.} In the context of deep learning, the output of each neuron is multiplied with a scalar which brings the necessity of having the ability to deal with weighted set of points. Algorithm~\ref{alg:main} can be extended easily to the case as generously detailed in Section~\ref{sec:ext_appendix_weighted} in the Appendix.  

\textbf{Dimensionality reduction.} All coreset-based pruning methods rely heavily on the dimensionality of the model's layers, as well as our method. To ensure sufficient pruning ratio, we apply either PCA, TSNE, MDS, or the JL transform on the weights of each layer prior to generating its coreset. 

\textbf{From weight to neuron pruning.} Most coreset-based pruning methods, e.g.,~\cite{baykal2018datadependent, mussay2021data}, first provide a scheme for (provable) weight pruning, which is then used as a stepping stone towards pruning neurons as follows.~\cite{baykal2018datadependent} first suggested coreset-based neuron pruning via the use of a generated controled set of queries to evaluate the importance of weights. Any neuron that has a maximal activation value lower or equal to zero, will be pruned from the network as its impact on the rest of the neurons is minimal. On the other hand~\cite{mussay2021data} altered the definition of sensitivity such that it takes into account the sensitivity of a neuron in a layer $\ell$ with respect to all the neurons in the layer $\ell+1$, which basically means that the sensitivity of each neuron is taken be the maximal sensitivity over every weight function (neuron in the next layer) defined by the layer. Hence, we follow the same logic for such method; see Section~\ref{ext:nuronp} in the supplementary material.


\textbf{From neuron to filter pruning.} Convolutions can be expressed as matrix multiplications, which enables our method to prune filters from any model as done by~\cite{liebenwein2020provable,mussay2021data,liebenwein2021compressing}. 

\section{Experimental Results}


In this section, we study various widely used network architectures and benchmark data-sets. Following~\cite{mussay2021data}, to test the robustness of our methods on each of the neuron and filter pruning tasks independently, two sets of experiments are conducted. The first focuses on pruning neurons (Section~\ref{sec:nueronpruning}) whereas the second focuses on pruning filters (Section~\ref{sec:filterpruning}), both via our coreset method. 

\textbf{The setting.} In all experiments we report the \emph{Pruning ratio} -- the percentage of the parameters that were removed from the original mode. Here, \emph{PR} stands for pruning ratio, \emph{FR} stands for floating-point reduction ratio and \emph{Err} -- the percentage of misclassified test instances of our method compared to coreset-based pruning methods and more. \emph{Baseline Err} is the error of the original uncompressed network, while  \emph{Pruned Err} is the classification error of the compressed model. In our experiment we compress and fine-tune the network once, no iterative pruning was applied, thus, the compared methods also satisfy this setting. Each experiment was conducted $5$ times, in the tables, we report for our method the best error achieved and we highlight in parentheses next to it the average error and standard deviation across the $5$ trails. In all of our experiments, the models are fine-tuned till convergence (after pruning). Implementation details are given in Section~\ref{sec:implementation_det} in the Appendix.

\textbf{Software/Hardware.} Our algorithms were implemented in Python 3.6~\cite{10.5555/1593511} using {Numpy}~\cite{oliphant2006guide}, and Pytorch~\cite{paszke2019pytorch}. Tests were performed on NVIDIA DGX A100 servers with 8 NVIDIA A100 GPUs each, fast InfiniBand interconnect and supporting infrastructure.

\textbf{Baselines.}  Our results are compared to (i) PFP ~\cite{liebenwein2020provable}, (ii) FT~\cite{li2016pruning}, (iii) SoftNet~\cite{hesoft}, (iv) ThiNet~\cite{luo2017thinet}, and (v) PvC~\cite{mussay2021data}, (vi) Soft Pruning~\cite{hesoft},  (vii) CCP~\cite{peng2019collaborative}, (viii) FPGM~\cite{he2019filter}, (ix) ThiNet-70, (x) ThiNet-50~\cite{luo2017thinet}, (xi) Pruning via Coresets (PvC)~\cite{mussay2021data}, (xii) Pruning from Scratch (PfS)~\cite{wang2020pruning}, and (xiii) Rethinking the value of network pruning (Rethink)~\cite{liu2018rethinking}.



\subsection{Neuron Pruning}\label{sec:nueronpruning}



We test our method on VGG16~\cite{simonyan2014very} using CIFAR10~\cite{krizhevsky2009learning}, and LeNet300-100 using MNIST~\cite{lecun1998gradient}. 
\begin{table}[htb!]
\caption{Pruning of LeNet-$300$-$100$ architecture on the MNIST dataset and of VGG-16 on the CIFAR-10 dataset. Here, we report the compression with respect to the fully connected layers only.}
\begin{center}
\adjustbox{max width=\linewidth}{
\centering
\begin{tabular}{c||c|c c c} 
\hline
Model & Method &  Baseline Err. (\%) &
Pruned Err. (\%) & PR (\%) \\ 
\hline
 \multirow{6}{2cm}{\centering LeNet-300-100} & PFP & $1.59$ & $2.00$ & $84.32$ \\
 & FT & $1.59$ & $1.94$ & $81.68$ \\
 & SoftNet & $1.59$ & $2.00$ & $81.69$ \\
  & ThiNet & $1.59$ & $12.17$ & $75.01$ \\
 & PvC & $2.16$ & $2.03$ & $90$ \\
 & \textbf{Our method ($90$)} & $\mathbf{2.07}$ & $\mathbf{1.98 \;  (2.02 \pm 0.04)}$ & $\mathbf{90}$ \\
 & \textbf{Our method ($92.6$)} & $\mathbf{2.07}$ & $\mathbf{2.64 \;  (2.74 \pm 0.1)}$ & $\mathbf{92.6}$ \\
  & \textbf{Our method ($94.6$)} & $\mathbf{2.07}$ & $\mathbf{3.51 \;  (3.58 \pm 0.07)}$ & $\mathbf{94.6}$ \\
 \hline 
 \multirow{2}{2cm}{\centering VGG-16} & PvC & $8.95$ & $8.16$ & $75$ \\
 & \textbf{Our method} & $\mathbf{5.9}$ & $\mathbf{5.9 \; (6.2 \pm 0.3)}$ & $\mathbf{90}$ \\
 \hline
\end{tabular}
}
\end{center}
\label{tab:neuronPruning}
\end{table}

\textbf{Discussion.} Table~\ref{tab:neuronPruning} present the results of LeNet$300$-$100$ and VGG$16$. Observe that in both architectures, our method outperformed the competing methods under the same compression scenarios. For example, we pruned roughly $90\%$ of the parameters of the LeNet-300-100 model while improving the accuracy of the original model. We witness a similar phenomena on the VGG16 model, where we pruned roughly $90\%$ of the parameters of the dense layers resulting in accuracy improvement. 
This confirms the insights in~\cite{baykal2018datadependent} that coresets help in improving the generalization properties of DNNs.

\subsection{Filter pruning}\label{sec:filterpruning}
We compressed the convolutional layers of (i) ResNet$50$~\cite{he2016deep} on ILSVRC-$2012$~\cite{deng2009imagenet}, (ii) ResNet$56$~\cite{he2016deep}, (iii) VGG$19$~\cite{simonyan2014very} on CIFAR$10$ and (iv) VGG$16$~\cite{simonyan2014very} on CIFAR$10$.

\begin{table}[h!]
\caption{Filter pruning results on different neural networks with respect to the CIFAR$10$ dataset.}
\begin{center}
\adjustbox{max width=\linewidth}{
\centering
\begin{tabular}{c||c| c c c c} 
 \hline
    Model & Method & Baseline Err. (\%) & Pruned Err. (\%)  & PR (\%) & FR (\%)\\
 \hline
\multirow{5}{2cm}{\centering VGG-19} & PfS & $6.4$ & $6.29$ & $52$ & NA\\
& Rethink & $6.5$ & $6.22$ & $80$ & NA\\
& Structured Pruning  & $6.33$& $6.20$ & $88$ & NA\\
 & PvC & $6.33$ & $6.02$ & $88$ & NA\\
 
 & \textbf{Our method ($88$)} & $\mathbf{6.33}$ & $\mathbf{5.85 \;  (6.03 \pm 0.18)}$ & $\mathbf{88}$ & NA \\
  
 & \textbf{Our method ($91.28$)} & $\mathbf{6.33}$ & $\mathbf{6.23 \;  (6.35 \pm 0.12 )}$ & $\mathbf{91.28}$ & NA \\
 \hline
 \multirow{5}{2cm}{\centering VGG-16} & ThiNet & $7.11$& $9.24$& $63.95$ & $64.02$ \\
 & FT & $7.11$ & $8.22$ & $80.09$ & $80.14$ \\
 & SoftNet & $7.11$ & $7.92$ & $63.95$ & $63.91$ \\
 & PFP ($94$) & $7.11$ & $7.61$ & $94.32$ & $85.03$ \\
 & PFP ($87$) & $7.11$ & $7.17$ & $87.06$ & $70.32$ \\
 & PFP ($80$) & $7.11$ & $7.06$ & $80.02$ & $59.21$\\
 & \textbf{Our method ($95.32$)} & $\mathbf{7.11}$ & $\mathbf{7.31 \;  (7.55 \pm 0.24)}$ & $\mathbf{95.32}$ & $\mathbf{85.09}$ \\
 & \textbf{Our method ($87$)} & $\mathbf{7.11}$ & $\mathbf{6.63 \;  (6.76 \pm 0.13)}$ & $\mathbf{87}$ & $\mathbf{68.2}$\\
 &\textbf{Our method ($79.53$)} & $\mathbf{7.11}$ & $\mathbf{6.3 \;  (6.38 \pm 0.08)}$ & $\mathbf{79.53}$ & $\mathbf{59.14}$\\
 \hline
 \multirow{6}{2cm}{\centering ResNet56} & ThiNet & $7.05$ & $8.433$ & $49.23$ & $49.74$\\ 
 & Channel Pruning & $7.2$ & $8.2$ & N/A & $50$   \\
 & AMC & $7.2$ & $8.1$ & $64.78$ & $50$\\
& CCP & $6.5$ & $6.42$ & $57$ & $52.6$\\
 & PvC & $6.21$ & $7.0$ & $55$ & N/A\\ 
 & \textbf{Our method} & $\mathbf{6.61}$ & $\mathbf{7.26 \;  (7.56 \pm 0.3)}$ & $\mathbf{63.95}$ & $\mathbf{50}$\\
 \hline
\end{tabular}
}
\end{center}
\label{tab:filter_pruning_cifar_10}
\end{table}

\begin{table}[h!]
\centering
\caption{Filter pruning of ResNet$50$ on ImageNet (ILSVRC-$2012$).}
\adjustbox{max width=\textwidth}{
\begin{tabular}{c|c c c c} 
 \hline
Method & Baseline Err. (\%) & Pruned Err. (\%) & PR (\%) & FR (\%)\\
 \hline
 PFP & $23.87$ & $24.79$ & $44.04$ & $30.05$ \\
 \hline 
 Soft Pruning & $23.85$ & $25.39$ & $49.35$ & $41.80$ \\
 \hline
 CCP & $23.85$ & $24.5$ & $56$ & $48.8$\\
 \hline
 FPGM & $23.85$ & $25.17$ & $62$ & $53.5$\\ 
 \hline
 ThiNet-$70$ & $27.72$ & $26.97$ & $33.72$ & $36.78$\\
 \hline
 ThiNet-$50$ & $27.72$ & $28.0$ & $51.5$ & $55.82$ \\ 
 \hline
 PvC & $23.78$ & $25.11$ & $62$ & N/A \\ 
 \hline
 \textbf{Our method} & $\mathbf{23.78}$ & $\mathbf{24.87 \;  (25.07 \pm 0.2)} $ & $\mathbf{62}$ & $\mathbf{61.5}$ \\ 
 \hline
\end{tabular}}
\label{tab:resnet50}
\end{table}

\textbf{Filter pruning of DNNs on CIFAR$\mathbf{10}$ and  ImageNet (ILSVRC-$2012$).} For Cifar$10$, we used PyTorch implementations of VGG$19$ and VGG$16$. We compressed both models using our approach by different compression rates  as shown in Table~\ref{tab:filter_pruning_cifar_10} with a comparison to other methods. For ImageNet, we compressed the baseline model of ResNet$50$~\cite{he2016deep} roughly by $62\%$ in terms of number of parameters. 
Table~\ref{tab:resnet50} provides comparison between our method and other baselines.

\textbf{Discussion. } As can be seen in Tables~\ref{tab:neuronPruning},~\ref{tab:filter_pruning_cifar_10}, and~\ref{tab:resnet50}, our method either outperforms the competing methods or achieves comparable results. As for the coreset methods, our algorithm achieves better result than all of them in this setting, e.g., we compressed $62\%$ of ResNet50 trained on ImageNet while incurring $1.09\%$ drop in accuracy, improving the recent coreset result of PvC~\cite{mussay2021data} for the same compression ratio, while PFP~\cite{liebenwein2020provable} compressed $44.04\%$ to achieve comparable results.

\section{Related work}
\label{sec:related}
DNNs can be compressed before training~\cite{ tanaka2020pruning, Wang2020Picking, lee2018snip}, during training~\cite{zhu2017prune,yu2018slimmable, kusupati2020soft}, or after training~\cite{singh2020woodfisher}. Furthermore, such procedures may also be repeated iteratively~\cite{renda2020comparing}. As previously noted pruning can be categorized into structured and unstructured pruning.

\textbf{Unstructured pruning.} Weight pruning~\cite{lecun1990optimal} techniques aim to reduce the number of weights in a layer while approximately preserving its output. 
Approaches of this type include the works of~\cite{lebedev2016fast,dong2017learning,iandola2016squeezenet,aghasi2017net,lin2017runtime}, where the desired sparsity is embedded as a constraint or via a regularizer into the training pipeline, and those of~\cite{Han15,renda2020comparing,guo2016dynamic}, where weights with absolute values below a threshold are removed. 
The approaches of~\cite{baykal2018datadependent, sipp2019} use a mini-batch of data points to approximate the influence of each parameter on the loss function. Other data-informed techniques include~\cite{gamboa2020campfire, Lin2020Dynamic, molchanov2016pruning, molchanov2019importance, yu2018nisp}. A thorough overview of recent pruning approaches is given by~\cite{gale2019state, blalock2020state}.
However, unlike our approach, weight-based pruning approaches generate sparse models instead of smaller ones thus requiring specialized hardware and sparse linear algebra libraries in order to speed up inference.

\textbf{Structured pruning.}
Pruning entire neurons and filters directly shrinks the network leading to smaller storage requirements and improved inference-time performance on any hardware~\cite{li2019learning,luo2018autopruner}. Lately, these approaches were investigated in many papers~\cite{liu2019metapruning,li2019learning,chen2020storage,he2019filter,dong2017more,kang2020operation,ye2020good, ye2018rethinking}. Usually, filters are pruned by assigning an importance score to each neuron/filter, either solely weight-based~\cite{he2017channel, he2018soft} or data-informed~\cite{yu2018nisp, liebenwein2020provable}, and removing those with a score below a threshold.
The procedure can be embedded into an iterative pruning scheme~\cite{renda2020comparing} that requires potentially expensive retrain cycles.

\textbf{Tensor decomposition. } 
Some of the work in DNN compression entails decomposing the layer into multiple smaller ones, e.g., via low-rank tensor decomposition~\cite{Denton14,jaderberg2014speeding,maalouf2020deep,kim2015compression, tai2015convolutional,ioannou2015training,alvarez2017compression,tukan2021no,yu2017compressing,lebedev2014speeding,liebenwein2021compressing}. 
Other approaches to tensor decomposition include weight sharing, random projections, and feature hashing~\cite{Weinberger09, arora2018stronger, shi2009hash, Chen15Hash, Chen15Fresh, ullrich2017soft}.
However, such techniques usually require expensive approximation algorithms or use heuristics since tensor decomposition is generally NP-hard.

\textbf{Coresets.} In the recent years, coresets got increasing attention, and where leveraged to compress the input datasets of many machine learning algorithms, improving there performance, e.g.,  regression~\cite{maalouf2022fast,huggins2016coresets,munteanu2018coresets,karnin2019discrepancy,nearconvex}, decision trees~\cite{jubran2021coresets}, matrix approximation~\cite{feldman2013turning, maalouf2019fast,feldman2010coresets,sarlos2006improved,maalouf2021coresets}, data discretization~\cite{maalouf2022coresets},  clustering~\cite{feldman2011scalable,gu2012coreset,lucic2015strong,bachem2018one,jubran2020sets, schmidt2019fair, tukan2022new}, $\ell_z$-regression~\cite{cohen2015lp, dasgupta2009sampling, sohler2011subspace}, \emph{SVM}~\cite{har2007maximum,tsang2006generalized,tsang2005core,tsang2005very,tukan2021coresets}, deep learning models~\cite{maalouf2021unified,baykal2018datadependent,liebenwein2020provable} and even for path planning in the field of robotics~\cite{tukan2022obstacle}. For extensive surveys on coresets, we refer the reader to~\cite{feldman2020core, phillips2016coresets, jubran2019introduction,maalouf2021introduction}.

\section{Conclusions and Future Work}\label{sec:conc}
In this paper, we provided a coreset-based pruning technique that hinges upon a combination of tools from convex geometry, while achieving SOTA results with respect to coreset-based structured pruning approaches on a variety of networks. Our main improvement is that our coreset is (training) data-independent and assumes a single assumption on the models weights. 

\textbf{Future work includes } (i) suggesting a coreset based budget allocation framework, to determine the (optimal) per layer prune ratio while achieving an overall desired compression rate, (ii) extending our coreset technique to other layers such as attention layers, and (iii) bridging the gap between coreset based pruning approaches and tensor-decomposition methods, as both techniques are theoretically supported by bounding the approximation error given specific compression rate, we can leverage these bounds to formulate the compression problem as an optimization problem which iterates between the two approaches to search for the local minimum.


\bibliographystyle{plain}
\bibliography{main}


\begin{enumerate}

\item For all authors...
\begin{enumerate}
  \item Do the main claims made in the abstract and introduction accurately reflect the paper's contributions and scope?
     \answerYes{}
  \item Did you describe the limitations of your work?
    \answerYes{See section~\ref{sec:conc}}
  \item Did you discuss any potential negative societal impacts of your work?
    \answerNA{}
  \item Have you read the ethics review guidelines and ensured that your paper conforms to them?
    \answerYes{}
\end{enumerate}

\item If you are including theoretical results...
\begin{enumerate}
  \item Did you state the full set of assumptions of all theoretical results?
    \answerYes{See Section~\ref{sec:method}}
        \item Did you include complete proofs of all theoretical results?
    \answerYes{See Sections~\ref{sec:analysis} and~\ref{sec:proofs}}
\end{enumerate}

\item If you ran experiments...
\begin{enumerate}
  \item Did you include the code, data, and instructions needed to reproduce the main experimental results (either in the supplemental material or as a URL)?
    \answerYes{}
  \item Did you specify all the training details (e.g., data splits, hyperparameters, how they were chosen)?
    \answerYes{}
        \item Did you report error bars (e.g., with respect to the random seed after running experiments multiple times)?
    \answerNo{}
        \item Did you include the total amount of compute and the type of resources used (e.g., type of GPUs, internal cluster, or cloud provider)?
    \answerYes{}
\end{enumerate}

\item If you are using existing assets (e.g., code, data, models) or curating/releasing new assets...
\begin{enumerate}
  \item If your work uses existing assets, did you cite the creators?
    \answerYes{}
  \item Did you mention the license of the assets?
    \answerYes{}
  \item Did you include any new assets either in the supplemental material or as a URL?
    \answerNo{}
  \item Did you discuss whether and how consent was obtained from people whose data you're using/curating?
    \answerNA{}
  \item Did you discuss whether the data you are using/curating contains personally identifiable information or offensive content?
    \answerNA{}
\end{enumerate}

\item If you used crowdsourcing or conducted research with human subjects...
\begin{enumerate}
  \item Did you include the full text of instructions given to participants and screenshots, if applicable?
    \answerNA{}
  \item Did you describe any potential participant risks, with links to Institutional Review Board (IRB) approvals, if applicable?
    \answerNA{}
  \item Did you include the estimated hourly wage paid to participants and the total amount spent on participant compensation?
    \answerNA{}
\end{enumerate}

\end{enumerate}

\newpage 
\appendix
\section{Computing the Carath\'{e}odory set}

\textbf{Overview of Algorithm~\ref{alg:cara}.} First, a convex combination of $v$ with respect to $P$ is formulated as a linear programming problem. This is done by reformulating the input set of points $P$ as a matrix denoted as $A \in \REAL^{(d+1) \times n}$ (see Line~\ref{alg:cara_A}). We then formulate the goal vector $b \in \REAL^{d+1}$ to be $v$ concatenated with an entry of $1$ which serves to make sure that the solution to 
\begin{equation*}
\begin{aligned}
& \underset{x \in \REAL^{n+1}}{\text{minimize}}
& & \mathbf{1}_{n+1}^T x \\
& \text{subject to}
& & Ax = b, \\
&&& x_i \in [0,1] \quad \forall i \in [d]
\end{aligned}
\end{equation*}
satisfies that $\sum\limits_{i \in [n]} x_i P_i = v$ and $\sum\limits_{i \in [n]} x_i = 1$. Solving this problem takes roughly $O^\ast\term{n^{\omega + o\term{1}}}$ where $\omega$ is the matrix multiplication exponent as elaborated in~\cite{cohen2021solving}; see Lines~\ref{alg:cara_b}--\ref{alg:cara_x}. We observe that $x$ from Line~\ref{alg:cara_x} might be dense, i.e., the number of non-zero entries exceeds $d+1$. To ensure that we have at max $d+1$ non-zero entries, we use Algorithm 1 of~\cite{maalouf2019fast} which aims to find a set of $d+1$ points, where their weighted average is the desired $v$ given the initial weight vector $x$; see Line~\ref{alg:cara_final}.

\begin{algorithm}
\caption{$\cara\term{v, P}$}
\SetKwInOut{Input}{Input}
\SetKwInOut{Output}{Output}
\SetAlgoLined
\DontPrintSemicolon
\Input{A point $v \in \REAL^d$ and a set $P \subseteq \REAL^d$ of $n$ points}
\Output{A subset $C \subseteq P$ of at max $d + 1$ points  such that $p \in \conv\term{C}$}
$A := \begin{bmatrix} \begin{bmatrix} P_1 \\ 1 \end{bmatrix} , \begin{bmatrix} P_2 \\ 1 \end{bmatrix}, \ldots, \begin{bmatrix} P_n \\ 1 \end{bmatrix}\end{bmatrix}$ \label{alg:cara_A} \tcc{$P_i$ here denotes the $i$th point in $P$}
$b := \begin{bmatrix} v \\ 1\end{bmatrix}$ \label{alg:cara_b} \;
$x := \argmin{\substack{x \in [0,\infty)^d \\ Ax = b}} \mathbf{1}_d^T x$ \label{alg:cara_x}\tcp{$\mathbf{1}_d$ denotes a $d$ dimensional vector of $1$s}
$C := \caraFast\term{P, x, d^2 + 2}$ \label{alg:cara_final} \tcc{See Algorithm 1 of \cite{maalouf2019fast}}
\label{alg:cara}
\Return $C$
\end{algorithm}
\section{Coreset-Related Technical Details}
\begin{definition}[VC-dimension~\cite{braverman2016new}]
\label{def:dimension}
For a query space $(P,w,\Q,f)$ and $r \in [0,\infty)$, we define 
\[
\RANGES(x,r) = \br{p \in P \mid w(p) f(p,x) \leq r},
\]
for every $x \in \Q$ and $r \geq 0$. The dimension of $(P, w, \Q, f)$ is the size $\abs{S}$ of the largest subset $S \subset P$ such that
\[
\abs{\br{S \cap \RANGES(x,r) \mid x \in \Q, r \geq 0 }} = 2^{\abs{S}},
\]
where $\abs{ A }$ denotes the number of points in $A$ for every $A \subseteq \REAL^d$.
\end{definition}

\subsection{Sensitivity Sampling Missing Details}\label{detailesonsensitity}
We want to use the sensitivity sampling framework to compute a coreset for a set of points $P$ in $\REAL^d$.

First, we need to bound the sensitivity of each point $p\in P$. The sensitivity pf a point $p\in P$ is defined as $s(p) = \sup_{x\in X}\frac{\phi(p,x)}{\sum_{q\in P}\phi(q,x)}$ where the denominator is not zero. 

Hence, for every $p\in P$, we wish to compute a number $s’(p)$, such that $s’(p)\geq s(p)$. Once the bound $s'(p)$ on the sensitivity $s(p)$ of each point $p$ is computed, we define $T=\sum_{p\in P} s’(p)$ as the total sensitivity. Now, to obtain a coreset, we can sample points according to the distribution $s’(p)/T$, i.e., we sample $m>0$ points from $P$, where at each sample, the point $p\in P$ is sampled i.i.d with probability $s’(p)/T$.  We also re-weight the sampled points to obtain a coreset.

As the bound $s’(p)$ (on $s(p)$) is tighter, the total sensitivity $T$ gets smaller, and then the coreset size (required number of sampled points) gets smaller, and vice versa.

\begin{theorem}[Restatement of Theorem 5.5 in~\cite{braverman2016new}]
\label{thm:braverman_coreset}
Let $P \subseteq \REAL^d$ be a set of $n$ points, $w : P \to [0,\infty)$ be a weight function , and let $f : P \times \REAL^d \to [0, \infty)$ be a loss function. For every $p \in P$ define the \emph{sensitivity} of $p$ as
$$
\sup_{x \in \REAL^d} \frac{w(p)f(p,x)}{\sum_{q \in P} w(q)f(q,x)},
$$
where the sup is over every $x \in \REAL^d$ such that the denominator is non-zero.
Let $s: P \to [0,1]$ be a function such that $s(p)$ is an upper bound on the sensitivity of $p$.
Let $t = \sum_{p \in P} s(p)$ and $d'$ be the~\emph{VC dimension} of the quadruple $\term{P,w,\REAL^d,f}$; see Definition~\ref{def:dimension}. Let $c \geq 1$ be a sufficiently large constant, $\varepsilon, \delta \in (0,1)$, and let $S$ be a random sample of $\abs{S} \geq \frac{ct}{\varepsilon^2}\left(d^\prime\log{t}+\log{\frac{1}{\delta}}\right)$ 
i.i.d points from $P$, such that every $p \in P$ is sampled with probability $\frac{s(p)}{t}$. Let $v(p) = \frac{tw(p)}{s(p)\abs{S}}$ for every $p \in S$. Then, with probability at least $1-\delta$, $(S,v)$ is an $\varepsilon$-coreset for $\term{P,w,\REAL^d,f}$.
\end{theorem}

\subsection{From Coresets to Approximating the Optimal Solution}\label{weak-coresets}

In optimization problems (or machine learning in general), the goal is usually to find a query that minimizes (or maximizes) some cost function. 
In the context of coresets, the goal is to find a small weighted subset such that for a given cost function, the cost of applying any solution (hypotheses/query) on the coreset approximates the cost of applying the same solution on the whole data. Since a coreset approximates the cost of every query, we do note that in many cases, coresets are applied for approximating the optimal solution. Specifically, solving the desired optimization problem on the whole data can be a hard problem when the time needed for such a solution is either polynomial or exponential in the size of the whole data, or when the required memory is too high. In this case, coresets can be leveraged, by computing the the optimal solution of fitting an $\eps$-coreset and applying it on the original data. If the computed coresets gives worst-case $(1+\eps)$-approximation error, then we provably $(1+4\eps)$-approximation towards the optimal cost of solving the optimization on the whole data (the proof is very easy, it is done by applying the triangle inequality few times). In other words, we can solve the problem on the coreset to obtain a solution $x^*$, and then apply $x^*$ to the whole data giving a good approximation for solving the problem from the beginning on the whole data. 


\section{Proofs of Technical Results}\label{sec:proofs}
\subsection{Proof of Lemma~\ref{lem:conv:ineq}}
\lemconvineq*
\begin{proof}
Since we can write $p$ as the convex combination of points $q \in A$ with weight $\alpha(q)$, we have
\[
\norm{p^T Y-v}_1=\norm{\term{\sum_{q\in A}\alpha(q)q^T}Y -v}_1.
\]
Moreover, we have $\sum_{q \in A} \alpha(q)=1$, so we can decompose $v$ into
\[
\norm{p^TY-v}_1= \norm{\sum_{q\in A}\alpha(q)\term{q^TY-v}}_1.
\]

By triangle inequality (or Jensen's inequality),
\[
\norm{p^TY-v}_1\leq \sum_{q\in A}\alpha(q) \norm{q^T Y-v}_1 \leq \max_{q\in A} \norm{q^TY-v}_1.
\]
\end{proof}

\subsection{Proof of Lemma~\ref{lem:l1_reg}}
\linflonereg*

\begin{proof}
Let $Y,z,P^\prime,K, S$ and $\mathrm{map}$ be defined as in Algorithm~\ref{alg:mainINF}. Since $P$ lies on a $r$-dimensional affine subspace, it holds that for every $p \in P$, $p = (p - z)YY^T + z$.
Note that $Y \in \REAL^{d \times r}$ is an orthogonal matrix (i.e., $Y^TY$ is the identity matrix in $\REAL^{r \times r}$) and $z \in \REAL^d$ denotes the translation of the affine subspace that $P$ lies on.

\textbf{Claim~\ref{clm:l_inf_1}. }$E(G,c)$ is the L\"{o}wner ellipsoid of $P^\prime$, which has $2r$ vertices. Since $V$ is the set of vertices of the shrunk form of $E(G,c)$ that is contained in $\conv\term{P^\prime}$, each point from $V$ can be represented as convex combination of $r+1$ points from $P^\prime$ by Carath\'{e}odory's Theorem. Then the number of points in $K$ is at most $2r\term{r+1}$, i.e., $\abs{K} \in O\term{r^2}$. Thus, $\abs{S} \in O\term{r^2}$ due to the fact that it can be constructed from $K$ through the use of $\mathrm{map}$. 

\textbf{Claim~\ref{clm:l_inf_2}. }First put $X \in \REAL^{d \times j}$ and $v \in \REAL^d$, and let $p \in \argsup{q \in P} \norm{\term{p-v}X}_1$. Since $S \subseteq P$, it holds that $\frac{\max_{q \in P} \norm{\term{q - v}X}_1}{\max_{q \in S} \norm{\term{q - v}X}_1} \geq 1$. Let $a := zYY^T + z - v$, we have
\begin{equation}
\label{eq:transform_between_dims}
\begin{split}
&\norm{\term{p - v}X}_1 = \norm{\term{\term{p - z}YY^T + z - v}X}_1^T = \norm{\term{pYY^T + a}X}_1 = \norm{\term{p^\prime Y^T + a}X}_1,
\end{split}
\end{equation}
where the first equality holds since $\mathrm{rank}(P)=r$, 
the second holds by definition of $a$, and the last equality holds by the construction of $p^\prime = pY$ at Line~\ref{alg_inf:Pprime} of Algorithm~\ref{alg:mainINF}.

Note that since $V$ is the set of vertices of $\frac{1}{r} \term{E(G,c) - c} + c$, by the definition of the L\"{o}wner ellipsoid, 
\begin{equation*}
\label{eq:inclusion1}
\begin{split}
V &\subseteq \conv\term{V} \subseteq \frac{1}{r}(E(G,c)-c)+ c \subseteq \conv\term{P^\prime} \subseteq E(G,c) \subseteq \conv\term{r^{1.5}\term{V - c} + c}.
\end{split}
\end{equation*}



Since $\conv\term{P'}$ enclose $\conv\term{V}$, and is enclosed by $\conv\term{r^{1.5}\term{V - c} + c}$, then there exists a point $q \in \conv\term{V}$ and $\gamma \in [0,1]$ such that $
p^\prime Y^T = \gamma qY^T + \term{1 - \gamma} \term{r^{1.5}\term{q - c} + c}Y^T,$
where by definition it holds that $r^{1.5}\term{q - c} + c \in \conv\term{r^{1.5}\term{V - c} + c}$, and $p^\prime = pY$.

By invoking Lemma~\ref{lem:conv:ineq}, we obtain that
\begin{align}
\label{eq:bounding_p}
\norm{\term{p^\prime Y^T + a}X}_1 &\leq \max\left\lbrace \norm{\term{q Y^T + a}X}_1, \norm{\term{r^{1.5}\term{q - c} + c + a}Y^TX}_1\right\rbrace. 
\end{align}

We note that
\begin{equation}
\label{eq:bounding_V}
\begin{split}
\norm{\term{qY^T + a}X}_1 &\leq \max_{\Tilde{q} \in V}\norm{\term{\Tilde{q} Y^T+ a}X}_1 \leq \max_{\Tilde{q} \in K}\norm{\term{\Tilde{q} Y^T + a}X}_1,
\end{split}
\end{equation}
where the first inequality follows from plugging $p := q$ and $A:=V$ into Lemma~\ref{lem:conv:ineq}, and the second inequality holds similarly since every point $V$ lies in $\conv\term{K}$. By invoking triangle inequality, we obtain that
\begin{align}
\label{eq:triangle_ineq}
&\norm{\term{r^{1.5}\term{\term{q - c} + c}Y^T + a}X}_1 
= \norm{\term{r^{1.5}qY^T + \term{1 - r^{1.5}}cY^T + a}X}_1 \\
&= \norm{\term{r^{1.5}qY^T + r^{1.5}a + \term{1 - r^{1.5}}cY^T + \term{1 - r^{1.5}}a}X}_1 \nonumber\\
&\leq r^{1.5}\norm{\term{qY^T+a}X}_1 +  \term{r^{1.5} - 1} \norm{\term{cY^T+a}X}_1, \nonumber
\end{align}
where the first equality follows by a simple rearrangement, and the second holds since $r^{1.5}a+ \term{1 - r^{1.5}}a = a$. Observe that $c \in \conv\term{V}$. Hence, by Lemma~\ref{lem:conv:ineq},
\begin{equation}
\label{eq:bounding_c}
\begin{split}
\norm{\term{cY^T+a}X}_1 \leq \max_{\Tilde{q} \in K}\norm{\term{\Tilde{q} Y^T+a}X}_1.
\end{split}
\end{equation}

By the construction of $S$, it holds that for every $p \in S$, $pY \in K$. Thus, combining~\eqref{eq:bounding_p},~\eqref{eq:bounding_V},~\eqref{eq:triangle_ineq} and~\eqref{eq:bounding_c} yields
$
\frac{1}{2r^{1.5}}\norm{\term{p^\prime Y^T + a}X}_1 \leq \max_{\Tilde{q} \in K} \norm{\term{\Tilde{q} Y^T + a}X}_1 = \max_{\Tilde{q} \in S} \norm{\term{\Tilde{q} Y Y^T +a}X}_1=\max_{\Tilde{q} \in S} \norm{\term{\Tilde{q} - v}X}_1$ where the last equality holds by~\eqref{eq:transform_between_dims}. This concludes Lemma~\ref{lem:l1_reg}.
\end{proof}

\subsection{Proof of Theorem~\ref{thm:relu_core}}
\begin{proof}
For space constraints let $\phi$ denote the $\mathrm{ReLU}$ function. To obtain a coreset, we first need to bound the sensitivity of each $p \in P$. Put $p \in P$ and let $x \in \argsup{x^\prime\in \REAL^d} \frac{\phi\term{p^Tx^\prime}}{\sum_{q \in P} \phi\term{q^Tx^\prime}}$ where the supremum is over every $x^\prime\in \REAL^d$ such that the denominator is not zero. Observe that 
\begin{equation*}
\begin{split}
&\frac{\sum\limits_{q \in P} \abs{q^Tx}}{\sum\limits_{q \in P} \phi\term{q^Tx}} =\frac{\sum\limits_{q \in P} \phi\term{q^Tx} + \sum\limits_{q \in P} \phi\term{-q^Tx}}{\sum\limits_{q \in P} \phi\term{q^Tx}}
= 1 + \frac{\sum\limits_{q \in P} \phi\term{-q^Tx}}{\sum\limits_{q \in P} \phi\term{q^Tx}} \leq 1 + \mu\term{P},
\end{split}
\end{equation*}
where the last inequality follows from Definition~\ref{def:complexity_mes}. Thus  
\begin{equation}
\label{eq:bound_relu_below}
\begin{split}
\sum_{q \in P} \phi\term{q^Tx} \geq \frac{1}{1 +\mu\term{P}} \sum_{q \in P} \abs{q^Tx}.
\end{split}
\end{equation}

Let $\beta= \term{1 + \mu\term{P}}$. We next observe that 
$
\frac{\phi\term{p^Tx}}{\sum_{q \in P}\phi\term{q^Tx}} \leq\frac{\abs{p^Tx}}{\sum_{q \in P}\phi\term{q^Tx}}\leq \beta\frac{\abs{p^Tx}}{\sum_{q \in P}\abs{q^Tx}},$
where the first inequality holds by properties of $\phi$, and the second is by~\eqref{eq:bound_relu_below}.
Hence the sensitivity of $p$ is bounded by
\begin{equation}
\label{eq:bound_sens_1}
\begin{split}
s(p) =  \frac{\phi\term{p^Tx}}{\sum_{q \in P} \phi\term{q^Tx}}\leq \beta \frac{\abs{p^Tx}}{\sum_{q \in P} \abs{q^Tx}}.
\end{split}
\end{equation}

Let $i$ be the iteration counter from Algorithm~\ref{alg:main} as defined in Line~\ref{alg1:line_i}, used in Line~\ref{alg1:used_i} and incremented in Line~\ref{alg1:line_advance_i}. The idea follows that of~\cite{varadarajan2012near} where points being discarded from $P$ at lower levels (smaller $i$'s) have higher sensitivity. This notion also resembles that of the \say{Onion sampling} of~\cite{jubran2020sets}. Now, assume that $p \in S_i$ at iteration $i$ of the while loop (Line~\ref{alg1:Si} of Algorithm~\ref{alg:main}). In this case, observe that by plugging $P:= Q \setminus \bigcup_{\hat{i} = 1}^{i - 1} S_{\hat{i}}$, $j=1$ and $v = -b\frac{x}{\norm{x}_2}$ into Lemma~\ref{lem:l1_reg}, we obtain a subset $S_i \subseteq Q$ such that
\begin{equation}
\label{eq:coreset_l_inf_guarantee}
\max_{q \in P} \abs{q^Tx} \leq \max_{q \in S} 2r^{1.5}\abs{q^Tx}.
\end{equation}

Thus for every $q \in S_i$,
\begin{equation*}
\begin{split}
\frac{s(p)}{\term{1 + \mu\term{P}}} &\leq \frac{\abs{p^Tx}}{\sum_{q \in P} \abs{q^Tx}} \leq \frac{\abs{p^Tx}}{\sum_{\hat{i} = 1}^i \max_{q \in S_{\hat{i}}}\abs{q^Tx}}\leq 2r^{1.5}\frac{\abs{p^Tx}}{\sum_{\hat{i} = 1}^i \abs{p^Tx}} = \frac{2r^{1.5}}{i},
\end{split}
\end{equation*}
where the first inequality is by~\eqref{eq:bound_sens_1}, the second inequality follows from the observation that $\br{\arg\max_{q \in S_{\hat{i}}} \abs{q^Tx}}_{\hat{i} =1}^i \subseteq P$ and the last inequality holds by~\eqref{eq:coreset_l_inf_guarantee}. 
Hence, we have obtained a bound on the sensitivity of each point $p \in P$. As for the total sensitivity, we observe that $t = \sum_{p \in P}s(p) \in O\term{\term{1 + \mu\term{P}} r^{3.5}\log{n}}$. 
Theorem~\ref{thm:braverman_coreset} states that to obtain an $\eps$-coreset with probability at least $1 - \delta$, the sample size $m$ must be $O\term{\frac{\mu\term{P} r^{3.5}\log{n}}{\eps^2} \term{d\term{\log\term{\mu\term{P} r\log{n}}} + \log\term{\frac{1}{\delta}}}}$.
\end{proof}

\section{Extension}
\subsection{Handling Weighted Sets of Points}
\label{sec:ext_appendix_weighted}
Similarly to~\cite{baykal2018datadependent,liebenwein2020provable}, we split the input data $P$ into two sets $P_+, P_- \subseteq P$ such that $P_+ = \br{p \in P \middle| w(p) \geq 0}$ while $P_- = \br{p \in P \middle| w(p) < 0}$. Following this step we call Algorithm~\ref{alg:mainGen} for each of the two sets with corresponding weights and corresponding sample sizes. To account for proper sample sizes, we split our theoretical bound of the required sample size for generating $\eps$-coreset into two terms for both $P_+$ and $P_-$ respectively, i.e., we formulate $m = m_{+} + m_{-}$ where $m_{+} = \frac{\abs{P_-}}{\abs{P}} m$ (similarly for $m_{-}$). 
Hence, we obtain an $\eps$-coreset for each of the query spaces $\term{P_{+},w,\REAL^d,\phi}$ and $\term{P_{-},w,\REAL^d,\phi}$.

\subsection{From Weight to Neuron Pruning}\label{ext:nuronp}

Most coreset-based pruning methods, e.g.,~\cite{baykal2018datadependent, mussay2021data}, first provide a scheme for (provable) weight pruning, 
which is then used as a stepping stone towards pruning neurons as follows. To prune neurons from a layer, post to computing the coreset-based weight pruning for each neuron, ideally we would have that at certain layer, for all neurons, the generated coreset contains the same set of neurons from previous layers, which in this case we can remove the neurons which are not in the coreset. However, such scenario is almost implausible. To deal with such problem, we discuss two ways to do so. The first method to deal with such problem is inspired by the technique used in~\cite{mussay2021data} which alters the definition of sensitivity such that it takes into account the sensitivity of a neuron in a layer $\ell$ with respect to all the neurons in the layer $\ell+1$, basically the sensitivity of each neuron is taken be the maximal sensitivity over every weight function (neuron in the next layer) defined by the layer. Hence, we follow the same logic for such method, more details .

\begin{algorithm}
\caption{$\gencoreset\term{P,w,m}$}
\label{alg:mainGen}
 \SetKwInOut{Input}{input}
\SetKwInOut{Output}{output}
\SetAlgoLined
\DontPrintSemicolon
\Input{A set $P\subseteq\mathbb{R}^d$ of $n$ points, a weight function $w(p) : P \to [0,\infty)$ and a sample size $m$}
\Output{ A weighted set $(C,u)$}
$Q := P$, $i := 1$, $C := \emptyset$\label{alg3:line_i}\;
\While{\label{alg3line:main_while_loop}$\abs{Q} \ge 2\mathrm{rank}\term{Q}^2$}{
$Q^\prime := \br{w(q) q \middle| q \in Q}$\;
$\mathrm{map}_w : Q^\prime \to Q$ a map that maps from $Q^\prime$ to $Q$\;
$S_i := \infcoreset\term{Q^\prime}$\label{alg3:Si}\;
\For{every $p \in S_i$\label{alg3:sens_start}}{
$s\term{\mathrm{map}_w\term{p}} := \frac{2r^{1.5}}{i}$ \label{alg3:used_i}\;
}\label{alg3:sens_end}
$Q :=  Q\setminus \br{\mathrm{map}_w\term{q} \middle| q \in S_i}$, $i := i+1$ \label{alg3:line_advance_i}\;
}
\For{every $p \in Q$}{
$s(p) := \frac{2r^{1.5}}{i}$\;
}
$t := \sum_{p\in P}s(p)$\;
$C:=$ an i.i.d sample of $m$ points from $P$, where each  $p\in P$ is sampled with probability $\frac{s(p)}{t}$. \label{alg3:start_coreset}\;
$u(p):= \frac{tw(p)}{m\cdot s(p)}$ for every $p\in C$ \label{alg3:end_coreset}\;
\Return$(C,u)$
\end{algorithm}

\subsection{Other Activation Functions}
\label{sec:ext_other_activations}
\cite{mai2021coresets} recently showed that there exists a family of functions $\mathcal{F}$ called \say{Nice hinge functions} such that for any query $x \in \REAL^d$ and a set of points $P \subseteq \REAL^d$, for any $\phi \in \mathcal{F}$, it holds that $\frac{\sum\limits_{p \in P} \phi\term{p^Tx}}{\sum\limits_{q \in P} \relu{q^Tx}}$ is bounded from below. Formally speaking, below is the definition of a \say{nice hinge function}.

\begin{definition}[Restatement of Definition~7 of~\cite{mai2021coresets}]
\label{def:nice_hinge}
We call $f : \REAL \to [0,\infty)$ an $\term{L,a_1,a_2}$-nice hinge function if for a fixed constant $L$, $a_1$ and $a_2$,
\begin{enumerate}[label=(\roman*)]
    \item $f$ is $L$-Lipschitz,
    \item $\abs{f(z) - \relu{z}} \leq a_1$ for all $z$, and 
    \item $f(z) \geq a_2$ for all $z \geq 0$.
\end{enumerate}
\end{definition}

As noted by~\cite{mai2021coresets}, the hinge and log losses are $(1,1,1)$-nice and $\term{1,\ln{2},\ln{2}}$-nice hinge functions respectively. Similarly, it is easy to show that the activation function $\phi(x) = \ln\term{1 + e^x}$ is $\term{1, \ln{2},\ln{2}}$-nice hinge functions. Following the same steps applied by~\cite{mai2021coresets}, we obtain that 
\[
\sum\limits_{p \in P}\phi\term{p^T x} \geq \min\br{\frac{a_2}{2a_1}, \frac{1}{2}}\frac{\sum\limits_{q \in P} \abs{q^Tx}}{\mu\term{P} + 1},
\]
where $\phi\term{\cdot}$ is a $\term{L,a_1,a_2}$ where $a_2$ is assumed to be positive. 

Unlike the $\mathrm{ReLU}$ activation function, to support for other activation functions, we need to restrict our query space to contain queries such that $\forall p \in P: \phi\term{p^Tx} \leq r\abs{p^Tx}$ where $r$ denotes the rank of $P$. Let $X^\prime$ denote the set of all such queries.

Hence, under this additional assumption, we obtain that for every $p \in P$ and $x \in X^\prime$
\[
\frac{\phi\term{p^Tx}}{\sum\limits_{q \in P} \phi\term{q^Tx}} \leq \frac{\term{1+ \mu\term{P}} \phi\term{p^Tx}}{\min\br{\frac{a_2}{2a_1},\frac{1}{2}} \sum\limits_{q \in P} \abs{q^Tx}} \leq \frac{\term{1+ \mu\term{P}} r\abs{p^Tx}}{\min\br{\frac{a_2}{2a_1},\frac{1}{2}} \sum\limits_{q \in P} \abs{q^Tx}}.
\]

Following the same steps done at the proof of Theorem~\ref{thm:relu_core}, we can generate an $\eps$-coreset with respect to $\term{P, \phi, X^\prime}$.

\section{Implementation Details}
\label{sec:implementation_det}

First observe that the \emph{map} function in Line~\ref{alg_inf:map} of Algorithm~\ref{alg:mainINF} is hard to implement if all we have is $Y$ and $P^\prime$ (see Lines~\ref{alg_inf:affine}--\ref{alg_inf:Pprime}) due to the fact that when the rank of $P$ is not $d$, then $Y$ becomes a singular matrix. A way around such problem (practically speaking yet also theoretically sound) is to reformulate $P$ and $P^\prime$ as matrices, where our $\ell_\infty$-coreset will now be regarded as a set of indices of the rows selected as the desired coreset. \textbf{Note} that while we relied on an ``accurate'' measure of the rank of points in Algorithm~\ref{alg:mainINF}, in our experiments, we used the rank function from \emph{Numpy}, and still produced favorable results. Furthermore, our algorithms work also when the input has full rank. In addition, we can still obtain an $\varepsilon$-coreset when using approximated algorithms for the rank computation problem, where the error associated with our coreset may increase. In this case, we can increase our coreset size to reduce our approximation error to be the original desired error.

\section{Complexity Measure - Clarification}

First, Note that while the complexity measure was first defined for construction of coresets with respect to the logistic regression problem, it also has been used for the ReLU regression problem (minimizing the sum of ReLU losses)~\cite{mai2021coresets}.

The complexity is expected to be small other than in some cases~\cite{munteanu2018coresets,mai2021coresets}. We also operate under the same assumption, i.e., the complexity measure is reasonably small. 

Specifically speaking, when given a set $P$ (expressing our neurons) containing points in $\mathbb{R}^3$ (for example), such that point in $P$ lie on a 2-dimensional affine subspace parallel to the $xy$-plane, notice that in our setting, the complexity measure is defined as the maximal value of an optimization problem involving our vectors and a set of queries $X:=\mathbb{R}^2 \times \left\lbrace 1 \right\rbrace$. 

The existing of a hyperplane whose normal in $X$ such that one point from $P$ can be separated from the rest of the points in $P$, leads to large complexity measure.

\begin{figure}[htb!]
    \centering
    \includegraphics[width=.7\linewidth]{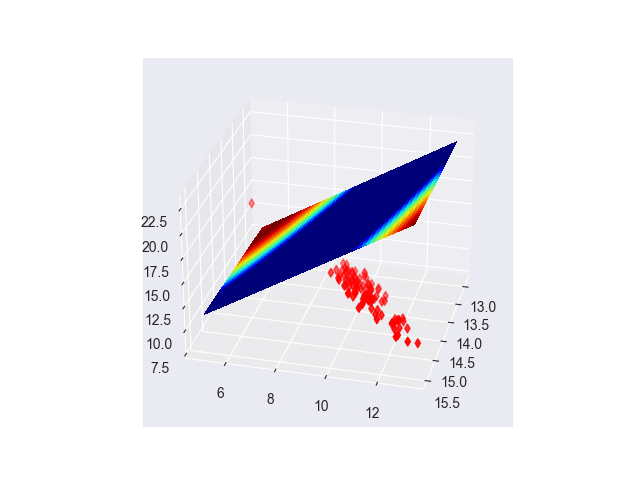}
    \caption{Linearly separable data leading to sufficiently large complexity measure.}
    \label{fig:badComplexity}
\end{figure}

In Figure~\ref{fig:badComplexity}, a clear separation can be made between one point and the rest leading to two sets of points: The first set contains one single point that has a positive dot product with the normal $x$ to the separating hyperplane, while the other set contains the remaining point each with negative dot product with $x$. This leads to a large complexity measure, and as the separating hyperplane gets closer and closer to the set containing the single point, the complexity measure increases, as it can tend to infinity.

On the other hand, when one can not separate a single point or minimal set of points from the rest of the data, we expect the complexity measure to be small; see Figure~\ref{fig:goodComplexity}.

\begin{figure}[htb!]
    \centering
    \includegraphics[width=.7\linewidth]{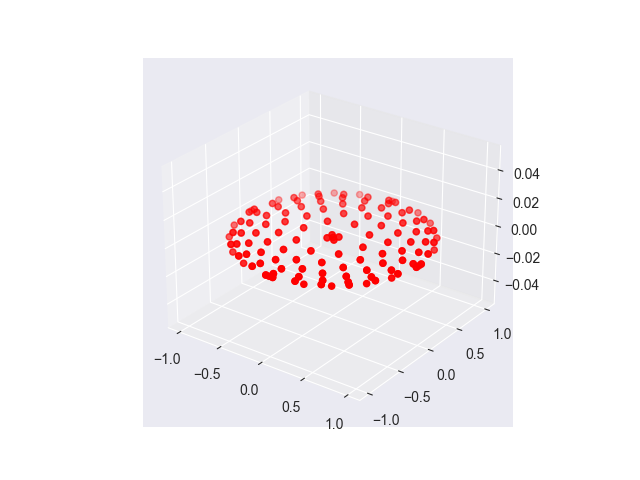}
    \caption{Non-linearly separable data leading to sufficiently small complexity measure.}
    \label{fig:goodComplexity}
\end{figure}

Notice that in Figure~\ref{fig:goodComplexity}, the input data is centered around the origin, which means that the data is not linearly separable. Thus leading to small complexity measure, since $x$ represents the normal to a hyperplane emerging from the origin. 

In fact, during our experiments, the complexity measure in the case of LeNet300-100 was around $15$ when the input data (matrix representing the neurons) had $300$ rows (points). 

It is common in coreset literature from a practical point of view, sample sizes that are smaller than the bound on the coreset size are being used. This is due to the fact that such bounds are pessimistic in nature. 

This motivated the choice of not incorporating the complexity measure in our sample size nor the sensitivity sampling since such a term will be eliminated when computing the sampling probability; see Theorem~\ref{thm:braverman_coreset}. Our experiments confirmed such an observation, i.e., our coreset lead to favorable results when the complexity measure was not incorporated in our computations, or when the sample size was much smaller than the bound on the coreset size.

The appendix in the supplementary material has been modified in light of this.

\section{Additional Experiments}
In all of our experiments, our hyper-parameters were drawn from~\cite{liebenwein2020provable}.


 

\subsection{The effect of fine-tuning}
In this experiment, we aim to show the effect of fine-tuning on our compressed model. Specifically, Figure~\ref{fig:against_ben_fine_tuning} shows VGG19's network accuracy over fine-tuning, where we start better than previous methods, followed by a slow incline in accuracy until we outperform previous models (around epoch $18$).

\label{sec:fine_tuning_exp}
\begin{figure}[htb!]
    \centering
    \includegraphics[width=\linewidth]{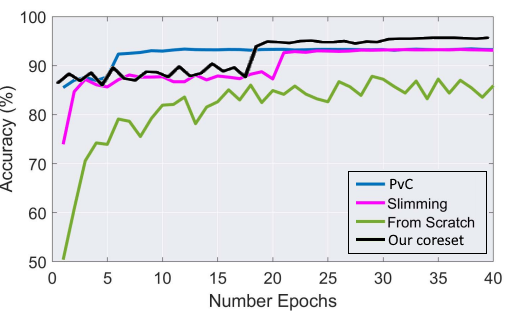}
    \caption{Accuracy of our proposed framework in comparison to previous methods and training the pruned network from scratch. The results above reflect the accuracy of VGG19 on CIFAR10.}
    \label{fig:against_ben_fine_tuning}
\end{figure}

\subsection{Sensitivity based distribution} At Figure~\ref{fig:probability}, we plot the sensitivity distribution of our sampling method in comparison to the sampling probabilities achieved by~\cite{mussay2021data}. Our advantage lies in the observation that our induced probability distribution entails longer tails, i.e., important point are scarce.

\begin{figure}[htb!]
    \centering
    \includegraphics[width=.7\textwidth]{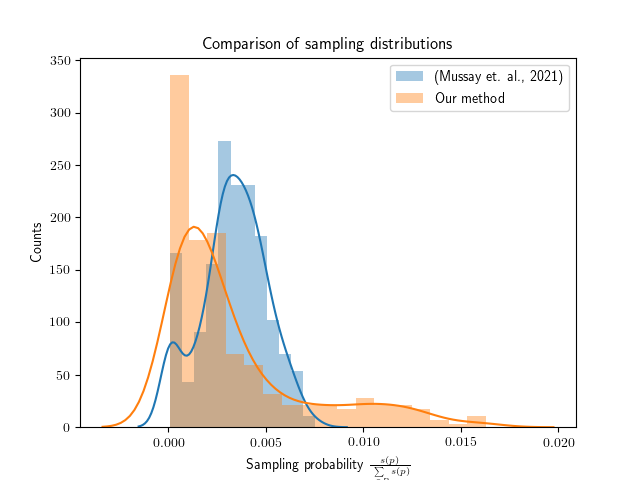}
    \caption{A comparison against~\cite{mussay2021data} with respect to the distribution of the sampling probabilities of weights of a single neuron at some layer of LeNet-$300$-$100$. Here the $x$-axis denotes the sampling probability of points, while the $y$-axis presents the number of points with certain probability.}
    \label{fig:probability}
    \end{figure}

\subsection{Comparison with PvC}
In this experiment, we aim to show the efficacy of our approach against that of~\cite{mussay2021data}. We considered a single neuron in LeNet-300-100 where we computed the average additive error of the cost of the coreset from the cost of taking all the samples (neurons from previous layer), over set of $1000$ queries. As shown in Figure~\ref{fig:us_neuron_mussay}, for very small coreset sizes, PvC~\cite{mussay2021data} attains smaller error, however as we increase the sample size, our coreset outperforms that of~\cite{mussay2021data}.

\begin{figure}[htb!]
    \centering
    \includegraphics[width=0.7\textwidth]{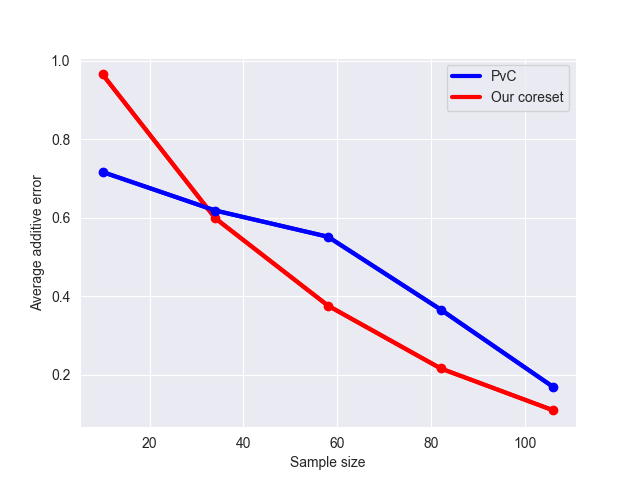}
    \caption{Average additive error of our coreset and that of~\cite{mussay2021data} on LeNet-300-100.}
    \label{fig:us_neuron_mussay}
\end{figure}

\end{document}